\documentclass[11pt,a4paper]{article}
\usepackage[hyperref]{acl2019}
\usepackage{times}
\usepackage{latexsym}

\usepackage{url}

\usepackage{amsmath}
\usepackage{amsthm}
\usepackage{amssymb}
\usepackage{thmtools}
\usepackage{mathtools}
\usepackage{bbm}

\usepackage{tikz}
\usetikzlibrary{arrows}

\theoremstyle{definition}
\newtheorem{definition}{Definition}[section]
\theoremstyle{plain}
\newtheorem{theorem}{Theorem}[section]
\newtheorem{lemma}{Lemma}[section]
\newtheorem{corollary}{Corollary}[theorem]

\newcommand{\CL}{\mathrm{CL}}
\newcommand{\SCL}{\mathrm{SCL}}

\newcommand{\RL}{\mathrm{RL}}

\newcommand{\SL}{\mathrm{SL}}

\usepackage{pgf}
\usepackage{tikz}
\usetikzlibrary{arrows,automata}

\newcommand{\nn}{\hat{\mathbbm{1}}}
\newcommand{\LSTM}{\mathrm{LSTM}}
\newcommand{\GRU}{\mathrm{GRU}}
\newcommand{\CNN}{\mathrm{CNN}}
\newcommand{\SRN}{\mathrm{SRN}}

\usepackage[utf8]{inputenc}
\usepackage{allrunes}
\newcommand{\confs}{\textarc{m}}

\DeclarePairedDelimiter\abs{\lvert}{\rvert}%
\DeclarePairedDelimiter\norm{\lVert}{\rVert}%

\makeatletter
\let\oldabs\abs
\def\abs{\@ifstar{\oldabs}{\oldabs*}}
\let\oldnorm\norm
\def\norm{\@ifstar{\oldnorm}{\oldnorm*}}
\makeatother

\newcommand\round[1]{\left[#1\right]}

\usepackage{expex}

\usepackage[b]{esvect}

\DeclareMathOperator{\softmax}{softmax}
\DeclareMathOperator{\ReLU}{ReLU}
\DeclareMathOperator{\attn}{attn}
\newcommand{\limN}{\lim_{N \rightarrow \infty}}

\newcommand{\bv}[1]{\mathbf{#1}}
\newcommand{\bm}[1]{\mathbf{#1}}

\newcommand{\autorefp}[1]{(\autoref{#1})}

\aclfinalcopy 

\title{Sequential Neural Networks as Automata}

\author{William Merrill\thanks{\; Work completed while the author was at Yale University.} \\
        Yale University, New Haven, CT, USA \\
        Allen Institute for Artificial Intelligence, Seattle, WA, USA \\
        \texttt{william.merrill@yale.edu}}

\date{\today}

\begin{document}

\maketitle
\begin{abstract}
This work attempts to explain the types of computation that neural networks can perform by relating them to automata. We first define what it means for a real-time network with bounded precision to accept a language. A measure of network memory follows from this definition. We then characterize the classes of languages acceptable by various recurrent networks, attention, and convolutional networks. We find that LSTMs function like counter machines and relate convolutional networks to the subregular hierarchy. Overall, this work attempts to increase our understanding and ability to interpret neural networks through the lens of theory. These theoretical insights help explain neural computation, as well as the relationship between neural networks and natural language grammar.
\end{abstract}

\section{Introduction}

In recent years, neural networks have achieved tremendous success on a variety of natural language processing (NLP) tasks. Neural networks employ continuous distributed representations of linguistic data, which contrast with classical discrete methods. While neural methods work well, one of the downsides of the distributed representations that they utilize is interpretability. It is hard to tell what kinds of computation a model is capable of, and when a model is working, it is hard to tell what it is doing.

This work aims to address such issues of interpretability by relating sequential neural networks to forms of computation that are more well understood. In theoretical computer science, the computational capacities of many different kinds of automata formalisms are clearly established. Moreover, the Chomsky hierarchy links natural language to such automata-theoretic languages \citep{chomsky1956three}. Thus, relating neural networks to automata both yields insight into what general forms of computation such models can perform, as well as how such computation relates to natural language grammar.

Recent work has begun to investigate what kinds of automata-theoretic computations various types of neural networks can simulate. \citet{weiss2018} propose a connection between long short-term memory networks (LSTMs) and counter automata. They provide a construction by which the LSTM can simulate a simplified variant of a counter automaton. They also demonstrate that LSTMs can learn to increment and decrement their cell state as counters in practice. \citet{peng2018rational}, on the other hand, describe a connection between the gating mechanisms of several recurrent neural network (RNN) architectures and weighted finite-state acceptors.

This paper follows \citet{weiss2018} by analyzing the expressiveness of neural network acceptors under asymptotic conditions. We formalize asymptotic language acceptance, as well as an associated notion of network memory. 
We use this theory to derive computation upper bounds and automata-theoretic characterizations for several different kinds of recurrent neural networks \autorefp{section:rnns}, as well as other architectural variants like attention \autorefp{section:attention} and convolutional networks (CNNs) \autorefp{section:cnns}. This leads to a fairly complete automata-theoretic characterization of sequential neural networks.

In \autoref{section:experiments}, we report empirical results investigating how well these asymptotic predictions describe networks with continuous activations learned by gradient descent. In some cases, networks behave according to the theoretical predictions, but we also find cases where there is gap between the asymptotic characterization and actual network behavior.

Still, discretizing neural networks using an asymptotic analysis builds intuition about how the network computes. Thus, this work provides insight about the types of computations that sequential neural networks can perform through the lens of formal language theory. In so doing, we can also compare the notions of grammar expressible by neural networks to formal models that have been proposed for natural language grammar.

\section{Introducing the Asymptotic Analysis} \label{asymptotic_analysis}

To investigate the capacities of different neural network architectures, we need to first define what it means for a neural network to accept a language. There are a variety of ways to formalize language acceptance, and changes to this definition lead to dramatically different characterizations.

In their analysis of RNN expressiveness, \citet{siegelmann1992turing} allow RNNs to perform an unbounded number of recurrent steps even after the input has been consumed. Furthermore, they assume that the hidden units of the network can have arbitrarily fine-grained precision. Under this very general definition of language acceptance, \citet{siegelmann1992turing} found that even a simple recurrent network (SRN) can simulate a Turing machine.

We want to impose the following constraints on neural network computation, which are more realistic to how networks are trained in practice \citep{weiss2018}:

\begin{enumerate}
    \item \textit{Real-time}: The network performs one iteration of computation per input symbol.
    \item \textit{Bounded precision}: The value of each cell in the network is representable by $O(\log n)$ bits on sequences of length $n$.
\end{enumerate}

Informally, a \textit{neural sequence acceptor} is a network which reads a variable-length sequence of characters and returns the probability that the input sequence is a valid sentence in some formal language. More precisely, we can write:

\begin{definition}[Neural sequence acceptor]
Let $\bm X$ be a matrix representation of a sentence where each row is a one-hot vector over an alphabet $\Sigma$. A neural sequence acceptor $\nn$ is a family of functions parameterized by weights $\theta$. For each $\theta$ and $\bm X$, the function $\nn^\theta$ takes the form
\begin{equation*}
    \nn^\theta : \bm X \mapsto p \in (0, 1) .
\end{equation*}
\end{definition}

\noindent In this definition, $\nn$ corresponds to a general architecture like an LSTM, whereas $\nn^\theta$ represents a specific network, such as an LSTM with weights that have been learned from data.

In order to get an acceptance decision from this kind of network, we will consider what happens as the magnitude of its parameters gets very large. Under these asymptotic conditions, the internal connections of the network approach a discrete computation graph, and the probabilistic output approaches the indicator function of some language \autorefp{fig:acceptance_example}.

\begin{figure} 
    \centering
    \includegraphics[scale=.2]{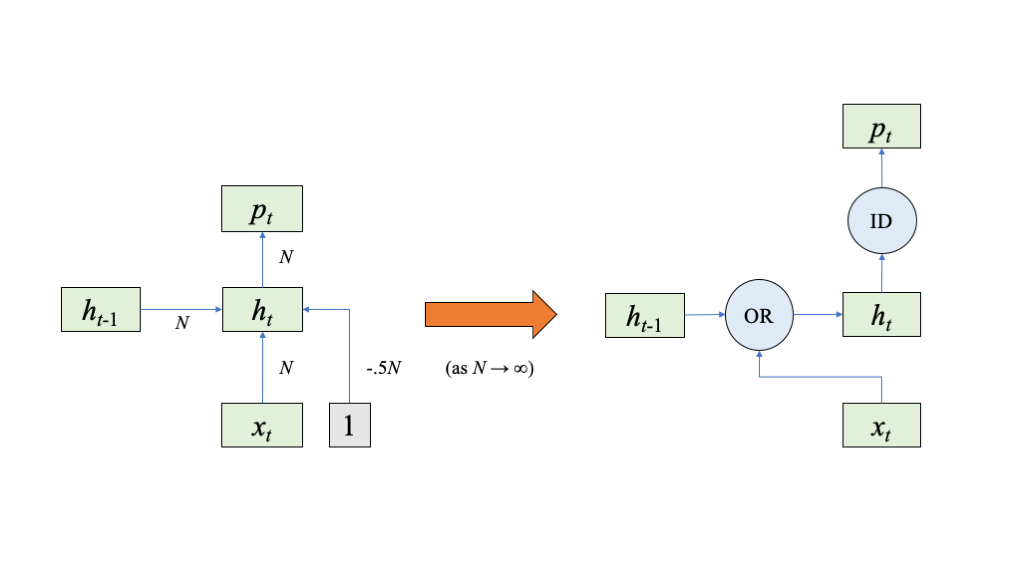}
    \caption{With sigmoid activations, the network on the left accepts a sequence of bits if and only if $x_t = 1$ for some $t$. On the right is the discrete computation graph that the network approaches asymptotically.}
    \label{fig:acceptance_example}
\end{figure}

\begin{definition}[Asymptotic acceptance]
\label{def:asymptotic_acceptance}
Let $L$ be a language with indicator function $\mathbbm{1}_L$. A neural sequence acceptor $\nn$ with weights $\theta$ asymptotically accepts $L$ if
\begin{equation*}
    \limN \nn^{N\theta} = \mathbbm{1}_L.
\end{equation*}
\end{definition}

\noindent Note that the limit of $\nn^{N\theta}$ represents the function that $\nn^{N\theta}$ converges to pointwise.\footnote{\url{https://en.wikipedia.org/wiki/Pointwise_convergence}}

Discretizing the network in this way lets us analyze it as an automaton. We can also view this discretization as a way of bounding the precision that each unit in the network can encode, since it is forced to act as a discrete unit instead of a continuous value. This prevents complex fractal representations that rely on infinite precision. We will see later that, for every architecture considered, this definition ensures that the value of every unit in the network is representable in $O(\log n)$ bits on sequences of length $n$.

It is important to note that real neural networks can learn strategies not allowed by the asymptotic definition. Thus, this way of analyzing neural networks is not completely faithful to their practical usage. In \autoref{section:experiments}, we discuss empirical studies investigating how trained networks compare to the asymptotic predictions. While we find evidence of networks learning behavior that is not asymptotically stable, adding noise to the network during training seems to make it more difficult for the network to learn non-asymptotic strategies.

Consider a neural network that asymptotically accepts some language. For any given length, we can pick weights for the network such that it will correctly decide strings shorter than that length (\autoref{thm:arbitrary_approximation}).

Analyzing a network's asymptotic behavior also gives us a notion of the network's memory. \citet{weiss2018} illustrate how the LSTM's additive cell update gives it more effective memory than the squashed state of an SRN or GRU for solving counting tasks. We generalize this concept of memory capacity as \textit{state complexity}. Informally, the state complexity of a node within a network represents the number of values that the node can achieve asymptotically as a function of the sequence length $n$. For example, the LSTM cell state will have $O(n^k)$ state complexity (\autoref{thm:lstm_memory_bound}), whereas the state of other recurrent networks has $O(1)$ (\autoref{thm:SRN_memory_bound}).

State complexity applies to a \textit{hidden state} sequence, which we can define as follows:

\begin{definition}[Hidden state] \label{def:hidden_state}
For any sentence $\bm X$, let $n$ be the length of $\bm X$. For $1 \leq t \leq n$, the $k$-length hidden state $\bv h_t$ with respect to parameters $\theta$ is a sequence of functions given by
\begin{equation*}
    \bv h^\theta_t : \bm X \mapsto \bv v_t \in \mathbb{R}^k .
\end{equation*}
\end{definition}

Often, a sequence acceptor can be written as a function of an intermediate hidden state. For example, the output of the recurrent layer acts as a hidden state in an LSTM language acceptor. In recurrent architectures, the value of the hidden state is a function of the preceding prefix of characters, but with convolution or attention, it can depend on characters occurring after index $t$.

The \textit{state complexity} is defined as the cardinality of the \textit{configuration set} of such a hidden state:

\begin{definition}[Configuration set]
For all $n$, the configuration set of hidden state $\bv{h}_n$ with respect to parameters $\theta$ is given by
\begin{equation*}
    M(\bv{h}^\theta_n) = \left\{ \limN \bv{h}_n^{N\theta}(\bm X) \mid n = \abs{\bm X} \right\} .
\end{equation*}
\noindent where $\abs{\bm X}$ is the length, or height, of the sentence matrix $\bm X$.
\end{definition}

\begin{definition}[Fixed state complexity]
For all $n$, the fixed state complexity of hidden state $\bv{h}_n$ with respect to parameters $\theta$ is given by
\begin{equation*}
    \confs(\bv{h}^\theta_n) = \abs{M(\bv{h}^\theta_n)} .
\end{equation*}
\end{definition}

\begin{definition}[General state complexity]
For all $n$, the general state complexity of hidden state $\bv{h}_n$ is given by
\begin{equation*}
    \confs(\bv{h}_n) = \max_\theta \confs(\bv{h}^\theta_n).
\end{equation*}
\end{definition}

To illustrate these definitions, consider a simplified recurrent mechanism based on the LSTM cell. The architecture is parameterized by a vector $\theta \in \mathbb{R}^2$. At each time step, the network reads a bit $x_t$ and computes
\begin{align}
    f_t &= \sigma(\theta_1 x_t) \\
    i_t &= \sigma(\theta_2 x_t) \\
    h_t &= f_t h_{t-1} + i_t .
\end{align}

When we set $\theta^+ = \langle 1, 1\rangle$, $h_t$ asymptotically computes the sum of the preceding inputs. Because this sum can evaluate to any integer between $0$ and $n$, $h_n^{\theta^+}$ has a fixed state complexity of
\begin{equation}
    \confs \left(h_n^{\theta^+} \right) = O(n) .
\end{equation}
\noindent However, when we use parameters $\theta^{\mathrm{Id}} = \langle -1, 1 \rangle$, we get a reduced network where $h_t = x_t$ asymptotically. Thus,

\begin{equation}
    \confs \left(h_n^{\theta^\mathrm{Id}} \right) = O(1).
\end{equation}

\noindent Finally, the general state complexity is the maximum fixed complexity, which is $O(n)$.

For any neural network hidden state, the state complexity is at most $2^{O(n)}$ (\autoref{thm:general_state_complexity}). This means that the value of the hidden unit can be encoded in $O(n)$ bits. Moreover, for every specific architecture considered, we observe that each fixed-length state vector has at most $O(n^k)$ state complexity, or, equivalently, can be represented in $O(\log n)$ bits.

Architectures that have exponential state complexity, such as the transformer, do so by using a variable-length hidden state. State complexity generalizes naturally to a variable-length hidden state, with the only difference being that $\bv h_t$ \autorefp{def:hidden_state} becomes a sequence of variably sized objects rather than a sequence of fixed-length vectors.

Now, we consider what classes of languages different neural networks can accept asymptotically. We also analyze different architectures in terms of state complexity. The theory that emerges from these tools enables better understanding of the computational processes underlying neural sequence models.

\section{Recurrent Neural Networks} \label{section:rnns}

As previously mentioned, RNNs are Turing-complete under an unconstrained definition of acceptance \citep{siegelmann1992turing}. The classical reduction of a Turing machine to an RNN relies on two unrealistic assumptions about RNN computation \citep{weiss2018}. First, the number of recurrent computations must be unbounded in the length of the input, whereas, in practice, RNNs are almost always trained in a real-time fashion. Second, it relies heavily on infinite precision of the network's logits. We will see that the asymptotic analysis, which restricts computation to be real-time and have bounded precision, severely narrows the class of formal languages that an RNN can accept.

\subsection{Simple Recurrent Networks} \label{section:srns}

The SRN, or Elman network, is the simplest type of RNN \citep{elman1990finding}:

\begin{definition}[SRN layer]
    \begin{equation}
    \bv h_t = \tanh(\bm W \bv x_t + \bm U \bv h_{t-1} + \bv b) .
    \end{equation}
\end{definition}

A well-known problem with SRNs is that they struggle with long-distance dependencies. One explanation of this is the vanishing gradient problem, which motivated the development of more sophisticated architectures like the LSTM \citep{hochreiter1997long}. Another shortcoming of the SRN is that, in some sense, it has less memory than the LSTM. This is because, while both architectures have a fixed number of hidden units, the SRN units remain between $-1$ and $1$, whereas the value of each LSTM cell can grow unboundedly \citep{weiss2018}. We can formalize this intuition by showing that the SRN has finite state complexity:

\begin{theorem}[SRN state complexity] \label{thm:SRN_memory_bound}
For any length $n$, the SRN cell state $\bv h_n \in \mathbb{R}^k$ has state complexity
\begin{equation*}
    \confs(\bv h_n) \leq 2^k = O(1) .
\end{equation*}
\end{theorem}

\begin{proof}
For every $n$, each unit of $\bv h_n$ will be the output of a $\tanh$. In the limit, it can achieve either $-1$ or $1$. Thus, for the full vector, the number of configurations is bounded by $2^k$.
\end{proof}

It also follows from \autoref{thm:SRN_memory_bound} that the languages asymptotically acceptable by an SRN are a subset of the finite-state (i.e. regular) languages. \autoref{thm:srn_lower_bound} provides the other direction of this containment. Thus, SRNs are equivalent to finite-state automata.

\begin{theorem}[SRN characterization] \label{thm:srn_reduction}
Let $L(\SRN)$ denote the languages acceptable by an SRN, and $\RL$ the regular languages. Then,
\begin{equation*}
    L(\SRN) = \RL .
\end{equation*}
\end{theorem}

This characterization is quite diminished compared to Turing completeness. It is also more descriptive of what SRNs can express in practice. We will see that LSTMs, on the other hand, are strictly more powerful than the regular languages.

\subsection{Long Short-Term Memory Networks}

An LSTM is a recurrent network with a complex gating mechanism that determines how information from one time step is passed to the next. Originally, this gating mechanism was designed to remedy the vanishing gradient problem in SRNs, or, equivalently, to make it easier for the network to remember long-term dependencies \citep{hochreiter1997long}. Due to strong empirical performance on many language tasks, LSTMs have become a canonical model for NLP.

\citet{weiss2018} suggest that another advantage of the LSTM architecture is that it can use its cell state as counter memory. They point out that this constitutes a real difference between the LSTM and the GRU, whose update equations do not allow it to increment or decrement its memory units. We will further investigate this connection between LSTMs and counter machines.

\begin{definition}[LSTM layer]
\begin{align}
    \bv f_t &= \sigma(\bm W^f \bv x_t + \bm U^f \bv h_{t-1} + \bv b^f) \\
    \bv i_t &= \sigma(\bm W^i \bv x_t + \bm U^i \bv h_{t-1} + \bv b^i) \\
    \bv o_t &= \sigma(\bm W^o \bv x_t + \bm U^o \bv h_{t-1} + \bv b^o) \\
    \bv{\tilde c_t} &= \tanh(\bm W^c \bv x_t + \bm U^c \bv h_{t-1} + \bv b^c) \\
    \bv c_t &= \bv f_t \odot \bv c_{t-1} + \bv i_t \odot \bv{\tilde c_t} \\
    \bv h_t &= \bv o_t \odot f(\bv c_t) . \label{eq:lstm_h}
\end{align}
\end{definition}

In \eqref{eq:lstm_h}, we set $f$ to either the identity or $\tanh$ \citep{weiss2018}, although $\tanh$ is more standard in practice. The vector $\bv h_t$ is the output that is received by the next layer, and $\bv c_t$ is an unexposed memory vector called the cell state.

\begin{theorem}[LSTM state complexity] \label{thm:lstm_memory_bound}
The LSTM cell state $\bv c_n \in \mathbb{R}^k$ has state complexity
\begin{equation*}
    \confs(\bv c_n) = O(n^k) .
\end{equation*}
\end{theorem}

\begin{proof}
At each time step $t$, we know that the configuration sets of $\bv f_t$, $\bv i_t$, and $\bv o_t$ are each subsets of $\{0, 1\}^k$. Similarly, the configuration set of $\bv{\tilde c_t}$ is a subset of $\{-1, 1\}^k$. This allows us to rewrite the elementwise recurrent update as
\begin{align}
    \limN [\bv c_t]_i &= \limN [\bv f_t]_i [\bv c_{t-1}]_i + [\bv i_t]_i [\bv{\tilde c_t}]_i \\
    &= \limN a [\bv c_{t-1}]_i + b
\end{align}
\noindent where $a \in \{0, 1\}$ and $b \in \{ -1, 0, 1 \}$.

Let $S_t$ be the configuration set of $[\bv c_t]_i$. At each time step, we have exactly two ways to produce a new value in $S_t$ that was not in $S_{t-1}$: either we decrement the minimum value in $S_{t-1}$ or increment the maximum value. It follows that
\begin{align}
|S_t| &= 2 + |S_{t-1}| \\
\implies |S_n| &= O(n) .
\end{align}
\noindent For all $k$ units of the cell state, we get
\begin{equation}
    \confs(\bv c_n) \leq |S_n|^k = O(n^k).
\end{equation}
\end{proof}

The construction in \autoref{thm:lstm_memory_bound} produces an automaton closely resembling a classical counter machine \citep{fischer1966turing, Fischer1968}.
Its restricted memory give us an upper bound on the expressive power of the LSTM:

\begin{theorem}[LSTM upper bound] \label{thm:lstm_upper_bound}
Let $\CL$ be the real-time
log-space languages.\footnote{\textbf{Revision:} A previous version stated ``real-time counter languages'' with no definition. Depending on the definition of counter languages, the claim may or may not hold. We clarify that $\CL$ is intended to be the ``log-space'' languages, where space is measured in bits. See \citet{merrill2020linguistic} for further discussion.}
Then,
\begin{equation*}
    L(\LSTM) \subseteq \CL.
\end{equation*}
\end{theorem}

\autoref{thm:lstm_upper_bound} constitutes a very tight upper bound on the expressiveness of LSTM computation. Asymptotically, LSTMs are not powerful enough to model even the deterministic context-free language $w\#w^R$.

\citet{weiss2018} show how the LSTM can simulate a simplified variant of the counter machine. Combining these results, we see that the asymptotic expressiveness of the LSTM falls somewhere between the general and simplified counter languages. This suggests counting is a good way to understand the behavior of LSTMs.

\subsection{Gated Recurrent Units}

The GRU is a popular gated recurrent architecture that is in many ways similar to the LSTM \citep{cho2014learning}. Rather than having separate forget and input gates, the GRU utilizes a single gate that controls both functions.

\begin{definition}[GRU layer]
\begin{align}
    \bv z_t &= \sigma(\bm W^z \bv x_t + \bm U^z \bv h_{t-1} + \bv b^z) \\
    \bv r_t &= \sigma(\bm W^r \bv x_t + \bm U^r \bv h_{t-1} + \bv b^r) \\
    \bv u_t &= \tanh \big( \bm W^u \bv x_t + \bm U^u(\bv r_t \odot \bv h_{t-1}) + \bv b^u \big) \label{eq:gru_squashed} \\
    \bv h_t &= \bv z_t \odot \bv h_{t-1} + (1 - \bv z_t) \odot \bv u_t .
\end{align}
\end{definition}

\citet{weiss2018} observe that GRUs do not exhibit the same counter behavior as LSTMs on languages like $a^nb^n$. As with the SRN, the GRU state is squashed between $-1$ and $1$ \eqref{eq:gru_squashed}. Taken together, Lemmas \ref{thm:gru_state_complexity} and \ref{thm:gru_lower_bound} show that GRUs, like SRNs, are finite-state.

\begin{theorem}[GRU characterization]
    \begin{equation*}
        L(\GRU) = \RL .
    \end{equation*}
\end{theorem}

\subsection{RNN Complexity Hierarchy}

Synthesizing all of these results, we get the following complexity hierarchy:
\begin{align}
    \RL &= L(\SRN) = L(\GRU) \\
    &\subset \SCL \subseteq L(\LSTM) \subseteq \CL .
\end{align}
\noindent Basic recurrent architectures have finite state, whereas the LSTM is strictly more powerful than a finite-state machine.

\section{Attention} \label{section:attention}

Attention is a popular enhancement to sequence-to-sequence (seq2seq) neural networks \citep{bahdanau2014neural, chorowski2015attention, luong2015effective}. Attention allows a network to recall specific encoder states while trying to produce output. In the context of machine translation, this mechanism models the alignment between words in the source and target languages. More recent work has found that ``attention is all you need'' \citep{vaswani2017attention, radford2018improving}. In other words, networks with only attention and no recurrent connections perform at the state of the art on many tasks.

An attention function maps a query vector and a sequence of paired key-value vectors to a weighted combination of the values. This lookup function is meant to retrieve the values whose keys resemble the query.

\begin{definition}[Dot-product attention] \label{def:attention}
For any $n$, define a query vector $\bv q \in \mathbb{R}^l$, matrix of key vectors $\bm K \in \mathbb{R}^{nl}$, and matrix of value vectors $\bm V \in \mathbb{R}^{nk}$. Dot-product attention is given by
    \begin{equation*}
        \attn(\bv q, \bm K, \bm V) = \softmax ( \bv q \bm K^T ) \bm V .
    \end{equation*}
\end{definition}

In \autoref{def:attention}, $\softmax$ creates a vector of similarity scores between the query $\bv q$ and the key vectors in $\bm K$. The output of attention is thus a weighted sum of the value vectors where the weight for each value represents its relevance.

In practice, the dot product $\bv q \bm K^T$ is often scaled by the square root of the length of the query vector \citep{vaswani2017attention}. However, this is only done to improve optimization and has no effect on expressiveness. Therefore, we consider the unscaled version.

In the asymptotic case, attention reduces to a weighted average of the values whose keys maximally resemble the query. This can be viewed as an $\arg \max$ operation.

\begin{theorem}[Asymptotic attention] \label{thm:asymptotic_attention}
    Let $t_1, .., t_m$ be the subsequence of time steps that maximize $\bv q \bv k_t$.%
    \footnote{To be precise, we can define a maximum over the similarity scores according to the order given by
    \begin{equation}
        f > g \iff \limN f(N) - g(N) > 0 .
    \end{equation}}
    Asymptotically, attention computes
    \begin{equation*}
        \limN \attn \left(\bv q, \bm K, \bm V\right) = \limN \frac{1}{m}\sum_{i=1}^m \bv v_{t_i} .
    \end{equation*}
\end{theorem}
\begin{corollary}[Asymptotic attention with unique maximum] \label{cor:injective_attention}
    If $\bv q \bv k_t$ has a unique maximum over $1 \leq t \leq n$, then attention asymptotically computes
    \begin{equation*}
        \limN \attn\left(\bv q, \bm K, \bm V\right) = \limN \arg \max_{\bv v_t} \bv q \bv k_t .
    \end{equation*}
\end{corollary}

Now, we analyze the effect of adding attention to an acceptor network. Because we are concerned with language acceptance instead of transduction, we consider a simplified seq2seq attention model where the output sequence has length $1$:

\begin{definition}[Attention layer] \label{thm:seq2seq_attention_layer}
Let the hidden state $\bv v_1, .., \bv v_n$ be the output of an encoder network where the union of the asymptotic configuration sets over all $\bv v_t$ is finite. We attend over $\bm V_t$, the matrix stacking $\bv v_1, .., \bv v_t$, by computing
\begin{equation*}
    \bv h_t = \attn(\bm W^q \bv v_t, \bm V_t, \bm V_t) .
\end{equation*}
\end{definition}

In this model, $\bv h_t$ represents a summary of the relevant information in the prefix $\bv v_1, .., \bv v_t$. The query that is used to attend at time $t$ is a simple linear transformation of $\bv v_t$.

In addition to modeling alignment, attention improves a bounded-state model by providing additional memory. By converting the state of the network to a growing sequence $\bv V_t$ instead of a fixed length vector $\bv v_t$, attention enables $2^{\Theta(n)}$ state complexity.

\begin{theorem}[Encoder state complexity] \label{thm:attention_state_complexity}
The full state of the attention layer has state complexity
\begin{equation*}
    \confs(\bm V_n) = 2^{\Theta(n)} .
\end{equation*}
\end{theorem}

The $O(n^k)$ complexity of the LSTM architecture means that it is impossible for LSTMs to copy or reverse long strings. The exponential state complexity provided by attention enables copying, which we can view as a simplified version of machine translation. Thus, it makes sense that attention is almost universal in machine translation architectures. The additional memory introduced by attention might also allow more complex hierarchical representations.

A natural follow-up question to \autoref{thm:attention_state_complexity} is whether this additional complexity is preserved in the attention summary vector $\bv h_n$. Attending over $\bm V_n$ does not preserve exponential state complexity. Instead, we get an $O(n^2)$ summary of $\bm V_n$.

\begin{theorem}[Summary state complexity] \label{thm:summary_complexity}
The attention summary vector has state complexity
\begin{equation*}
    \confs(\bv h_n) = O(n^2) .
\end{equation*}
\end{theorem}

With minimal additional assumptions, we can show a more restrictive bound: namely, that the complexity of the summary vector is finite. \autoref{sec:attention_results} discusses this in more detail.

\section{Convolutional Networks} \label{section:cnns}

While CNNs were originally developed for image processing \citep{krizhevsky2012imagenet}, they are also used to encode sequences. One popular application of this is to build character-level representations of words \citep{kim2016character}. Another example is the capsule network architecture of \citet{zhao2018investigating}, which uses a convolutional layer as an initial feature extractor over a sentence.

\begin{definition}[CNN acceptor]
\begin{align}
\bv h_t &= \tanh \big( \bm W^h (\bv x_{t-k} \Vert .. \Vert \bv x_{t+k}) + \bv b^h \big) \label{eq:conv} \\
\bv h_+ &= \mathrm{maxpool}(\bm H) \label{eq:pooling} \\
p &= \sigma(\bm W^a \bv h_+ + \bv b^a) . \label{eq:ff}
\end{align}
\end{definition}

In this network, the $k$-convolutional layer \eqref{eq:conv} produces a vector-valued sequence of outputs. This sequence is then collapsed to a fixed length by taking the maximum value of each filter over all the time steps \eqref{eq:pooling}.

The CNN acceptor is much weaker than the LSTM. Since the vector $\bv h_t$ has finite state, we see that $L(\CNN) \subseteq \RL$. Moreover, simple regular languages like $a^*ba^*$ are beyond the CNN \autorefp{thm:cnn_counterexample}. Thus, the subset relation is strict.

\begin{theorem}[CNN upper bound]
\begin{equation*}
    L(\CNN) \subset \RL .
\end{equation*}
\end{theorem}

So, to arrive at a characterization of CNNs, we should move to subregular languages. In particular, we consider the strictly local languages \citep{rogers2011aural}.

\begin{theorem}[CNN lower bound] \label{thm:conv_strictly_local}
Let $\SL$ be the strictly local languages. Then,
\begin{equation*}
    \SL \subseteq L(\CNN).
\end{equation*}
\end{theorem}

Notably, strictly local formalisms have been proposed as a computational model for phonological grammar \citep{heinz2011tier}. We might take this to explain why CNNs have been successful at modeling character-level information.

However, \citet{heinz2011tier} suggest that a generalization to the tier-based strictly local languages is necessary to account for the full range of phonological phenomena. Tier-based strictly local grammars can target characters in a specific tier of the vocabulary (e.g. vowels) instead of applying to the full string. While a single convolutional layer cannot utilize tiers, it is conceivable that a more complex architecture with recurrent connections could.

\section{Empirical Results} \label{section:experiments}

In this section, we compare our theoretical characterizations for asymptotic networks to the empirical performance of trained neural networks with continuous logits.\footnote{\url{https://github.com/viking-sudo-rm/nn-automata}}

\subsection{Counting}

\begin{table*}
    \centering
    \begin{tabular}{|cc|cc|cc|}
        \hline
        & & \multicolumn{2}{c|}{No Noise} & \multicolumn{2}{c|}{Noise} \\
        & $\confs$ & Acc & Acc on $c$ & Acc & Acc on $c$ \\
        \hline
        SRN & $O(1)$ & 100.0 & 100.0 & 49.9 & 100.0 \\
        GRU & $O(1)$ & 99.9 & 100.0 & 53.9 & 100.0 \\
        LSTM & $O(n^k)$ & 99.9 & 100.0 & 99.9 & 100.0\\
        \hline
    \end{tabular}
    \caption{Generalization performance of language models trained on $a^nb^nc$. Each model has $2$ hidden units.}
    \label{fig:counting_results}
\end{table*}

The goal of this experiment is to evaluate which architectures have memory beyond finite state. We train a language model on $a^nb^nc$ with $5 \leq n \leq 1000$ and test it on longer strings $(2000 \leq n \leq 2200)$. Predicting the $c$ character correctly while maintaining good overall accuracy requires $O(n)$ states. The results reported in \autoref{fig:counting_results} demonstrate that all recurrent models, with only two hidden units, find a solution to this task that generalizes at least over this range of string lengths.

\citet{weiss2018} report failures in  attempts to train SRNs and GRUs to accept counter languages, unlike what we have found. We conjecture that this stems not from the requisite memory, but instead from the different objective function we used. Our language modeling training objective is a robust and transferable learning target \citep{radford2019language}, whereas sparse acceptance classification might be challenging to learn directly for long strings.

\citet{weiss2018} also observe that LSTMs use their memory as counters in a straightforwardly interpretable manner, whereas SRNs and GRUs do not do so in any obvious way. Despite this, our results show that SRNs and GRUs are nonetheless able to implement generalizable counter memory while processing strings of significant length. Because the strategies learned by these architectures are not asymptotically stable, however, their schemes for encoding counting are less interpretable.

\subsection{Counting with Noise}

In order to abstract away from asymptotically unstable representations, our next experiment investigates how adding noise to an RNN's activations impacts its ability to count. For the SRN and GRU, noise is added to $\bv h_{t-1}$ before computing $\bv h_t$, and for the LSTM, noise is added to $\bv c_{t-1}$. In either case, the noise is sampled from the distribution $N(0, 0.1^2)$.

The results reported in the right column of \autoref{fig:counting_results} show that the noisy SRN and GRU now fail to count, whereas the noisy LSTM remains successful. Thus, the asymptotic characterization of each architecture matches the capacity of a trained network  when a small amount of noise is introduced.

From a practical perspective, training neural networks with Gaussian noise is one way of improving generalization by preventing overfitting \citep{Bishop:1995:TNE:211171.211185, DBLP:conf/nips/NohYMH17}. From this point of view, asymptotic characterizations might be more descriptive of the generalization capacities of regularized neural networks of the sort necessary to learn the patterns in natural language data as opposed to the unregularized networks that are typically used to learn the patterns in carefully curated formal languages.

\subsection{Reversing}

Another important formal language task for assessing network memory is string reversal. Reversing requires remembering a $\Theta(n)$ prefix of characters, which implies $2^{\Theta(n)}$ state complexity.

We frame reversing as a seq2seq transduction task, and compare the performance of an LSTM encoder-decoder architecture to the same architecture augmented with attention. We also report the results of \citet{hao2018context} for a stack neural network (StackNN), another architecture with $2^{\Theta(n)}$ state complexity (\autoref{thm:stack_state_complexity}).


\begin{table}
    \centering
    \begin{tabular}{|cc|cc|}
        \hline
        & $\confs$ & Val Acc & Gen Acc \\
        \hline
        LSTM & $O(n^k)$ & 94.0 & 51.6 \\
        LSTM-Attn & $2^{\Theta(n)}$ & 100.0 & 51.7 \\
        \hline
        LSTM & $O(n^k)$ & 92.5 & 73.3 \\
        StackNN & $2^{\Theta(n)}$ & 100.0 & 100.0\\
        \hline
    \end{tabular}
    \caption{Max validation and generalization accuracies on string reversal over 10 trials. The top section shows our seq2seq LSTM with and without attention. The bottom reports the LSTM and StackNN results of \citet{hao2018context}. Each LSTM has $10$ hidden units.}
    \label{table:extreme_reverse}
\end{table}


Following \citet{hao2018context}, the models were trained on $800$ random binary strings with length $\sim N(10, 2)$ and evaluated on strings with length $\sim N(50, 5)$. As can be seen in \autoref{table:extreme_reverse}, the LSTM with attention achieves 100.0\% validation accuracy, but fails to generalize to longer strings. In contrast, \citet{hao2018context} report that a stack neural network can learn and generalize string reversal flawlessly. In both cases, it seems that having $2^{\Theta(n)}$ state complexity enables better performance on this memory-demanding task. However, our seq2seq LSTMs appear to be biased against finding a strategy that generalizes to longer strings.

\section{Conclusion}

We have introduced asymptotic acceptance as a new way 
to characterize neural networks as automata of different sorts. It provides a useful and generalizable tool for building intuition about how a network works, as well as for comparing the formal properties of different architectures. Further, by combining asymptotic characterizations with existing results in mathematical linguistics, we can better assess the suitability of different architectures for the representation of natural language grammar.

We observe empirically, however, that this discrete analysis fails to fully characterize the range of behaviors expressible by neural networks. In particular, RNNs predicted to be finite-state solve a task that requires more than finite memory. On the other hand, introducing a small amount of noise into a network's activations seems to prevent it from implementing non-asymptotic strategies. Thus, asymptotic characterizations might be a good model for the types of generalizable strategies that noise-regularized neural networks trained on natural language data can learn.

\section*{Acknowledgements}

Thank you to Dana Angluin and Robert Frank for their insightful advice and support on this project.

\bibliography{acl2019}
\bibliographystyle{acl_natbib}

\appendix

\section{Asymptotic Acceptance and State Complexity} \label{sec:acceptance_complexity}

\begin{theorem}[Arbitary approximation] \label{thm:arbitrary_approximation}
Let $\nn$ be a neural sequence acceptor for $L$. For all $m$, there exist parameters $\theta_m$ such that, for any string $\bv x_1, .., \bv x_n$ with $n < m$,
\begin{equation*}
    \round{ \nn^{\theta_m}(\bm X) } = \mathbbm{1}_L(\bm X)
\end{equation*}
\noindent where $\round{\cdot}$ rounds to the nearest integer.
\end{theorem}

\begin{proof}
Consider a string $\bm X$. By the definition of asymptotic acceptance, there exists some number $M_{\bm X}$ which is the smallest number such that, for all $N \geq M_{\bm X}$,
\begin{align}
    \abs{\nn^{N\theta}(\bv X) - \mathbbm{1}_L(\bv X)} &< \frac{1}{2} \\
    \implies \round{ \nn^{N\theta}(\bm X) } &= \mathbbm{1}_L(\bm X) .
\end{align}
Now, let $X_m$ be the set of sentences $\bm X$ with length less than $m$. Since $X_m$ is finite, we pick $\theta_m$ just by taking
\begin{equation}
    \theta_m = \max_{\bm X \in X_m} M_{\bm X} \theta .
\end{equation}
\end{proof}

\begin{theorem}[General bound on state complexity] \label{thm:general_state_complexity}
Let $\bv h_t$ be a neural network hidden state. For any length $n$, it holds that

\begin{equation*}
    \confs(\bv{h}_n) = 2^{O(n)} .
\end{equation*}
\end{theorem}

\begin{proof}
The number of configurations of $\bv{h}_n$ cannot be more than the number of distinct inputs to the network. By construction, each $\bv x_t$ is a one-hot vector over the alphabet $\Sigma$. Thus, the state complexity is bounded according to

\begin{equation*}
\confs(\bv{h}_n) \leq \abs{\Sigma}^n = 2^{O(n)} .
\end{equation*}
\end{proof}

\section{SRN Lemmas} \label{sec:srn_proofs}

\begin{lemma}[SRN lower bound] \label{thm:srn_lower_bound}
\begin{equation*}
    \RL \subseteq L(\SRN) .
\end{equation*}
\end{lemma}

\begin{proof}
We must show that any language acceptable by a finite-state machine is SRN-acceptable. We need to asymptotically compute a representation of the machine's state in $\bv h_t$. We do this by storing all values of the following finite predicate at each time step:
\begin{equation} \label{eq:eth_srn}
    \eth_t(i, \alpha) \iff q_{t-1}(i) \wedge x_t = \alpha
\end{equation}
\noindent where $q_t(i)$ is true if the machine is in state $i$ at time $t$.

Let $F$ be the set of accepting states for the machine, and let $\delta^{-1}$ be the inverse transition relation. Assuming $\bv h_t$ asymptotically computes $\eth_t$, we can decide to accept or reject in the final layer according to the linearly separable disjunction
\begin{equation} \label{eq:srn_acceptance}
    a_t \iff \bigvee_{i \in F} \bigvee_{\langle j, \alpha \rangle \in \delta^{-1}(i)} \eth_t(j, \alpha) .
\end{equation}

We now show how to recurrently compute $\eth_t$ at each time step. By rewriting $q_{t-1}$ in terms of the previous $\eth_{t-1}$ values, we get the following recurrence:
\begin{equation}
    \eth_t(i, \alpha) \iff x_t = \alpha \wedge \bigvee_{\langle j, \beta \rangle \in \delta^{-1}(i)} \eth_t(j, \beta) .
\end{equation}
\noindent Since this formula is linearly separable, we can compute it in a single neural network layer from $\bv x_t$ and $\bv h_{t-1}$.

Finally, we consider the base case. We need to ensure that transitions out of the initial state work out correctly at the first time step. We do this by adding a new memory unit $f_t$ to $\bv h_t$ which is always rewritten to have value $1$. Thus, if $f_{t-1} = 0$, we can be sure we are in the initial time step. For each transition out of the initial state, we add $f_{t-1} = 0$ as an additional term to get
\begin{multline} \label{eq:srn_base_case}
    \eth_t(0, \alpha) \iff x_t = \alpha \; \wedge \\
    \big( f_{t-1} = 0 \; \vee \bigvee_{\langle j, \beta \rangle \in \delta^{-1}(0)} \eth_t(j, \beta) \big) .
\end{multline}
\noindent This equation is still linearly separable and guarantees that the initial step will be computed correctly.
\end{proof}

\section{GRU Lemmas} \label{sec:gru_results}

These results follow similar arguments to those in \autoref{section:srns} and \autoref{sec:srn_proofs}.

\begin{lemma}[GRU state complexity] \label{thm:gru_state_complexity}
The GRU hidden state has state complexity
\begin{equation*}
    \confs(\bv h_n) = O(1) .
\end{equation*}
\end{lemma}

\begin{proof}
The configuration set of $\bv z_t$ is a subset of $\{0, 1\}^k$. Thus, we have two possibilities for each value of $[\bv h_t]_i$: either $[\bv h_{t-1}]_i$ or $[\bv u_t]_i$. Furthermore, the configuration set of $[\bv u_t]_i$ is a subset of $\{-1, 1\}$. Let $S_t$ be the configuration set of $[\bv h_t]_i$. We can describe $S_t$ according to
\begin{align}
    S_0 &= \{ 0 \} \\
    S_t &\subseteq S_{t-1} \cup \{-1, 1\} .
\end{align}
\noindent This implies that, at most, there are only three possible values for each logit: $-1$, $0$, or $1$. Thus, the state complexity of $\bv h_n$ is
\begin{equation}
    \confs(\bv h_n) \leq 3^k = O(1) .
\end{equation}
\end{proof}

\begin{lemma}[GRU lower bound] \label{thm:gru_lower_bound}
\begin{equation*}
    \RL \subseteq L(\GRU).
\end{equation*}
\end{lemma}

\begin{proof}
We can simulate a finite-state machine using the $\eth$ construction from \autoref{thm:srn_reduction}. We compute values for the following predicate at each time step:
\begin{equation} \label{eq:gru_eth}
    \eth_t(i, \alpha) \iff x_t = \alpha \wedge \bigvee_{\langle j, \beta \rangle \in \delta^{-1}(i)} \eth_{t-1}(j, \beta).
\end{equation}
\noindent Since \eqref{eq:gru_eth} is linearly separable, we can store $\eth_t$ in our hidden state $\bv h_t$ and recurrently compute its update. The base case can be handled similarly to \eqref{eq:srn_base_case}. A final feedforward layer accepts or rejects according to \eqref{eq:srn_acceptance}.
\end{proof}

\section{Attention Lemmas} \label{sec:attention_results}

\begin{theorem}[\autoref{thm:asymptotic_attention} restated]
    Let $t_1, .., t_m$ be the subsequence of time steps that maximize $\bv q \bv k_t$. Asymptotically, attention computes
    \begin{equation*}
        \limN \attn \left(\bv q, \bm K, \bm V\right) =  \limN \frac{1}{m}\sum_{i=1}^m \bv v_{t_i} .
    \end{equation*}
\end{theorem}

\begin{proof}
Observe that, asymptotically, $\mathrm{softmax}(\bv u)$ approaches a function
\begin{equation}
    \limN \softmax(N\bv u)_t =
    \begin{cases}
    \frac{1}{m} & \textrm{if} \; u_t = \max(\bv u)\\
    0 & \textrm{otherwise.}
    \end{cases}
\end{equation}
\noindent Thus, the output of the attention mechanism reduces to the sum
\begin{equation}
    \limN \sum_{i=1}^m \frac{1}{m} \bv v_{t_i} .
\end{equation}
\end{proof}

\begin{lemma}[\autoref{thm:attention_state_complexity} restated]
The full state of the attention layer has state complexity
\begin{equation*}
    \confs(\bm V_n) = 2^{\Theta(n)} .
\end{equation*}
\end{lemma}

\begin{proof}
By the general upper bound on state complexity \autorefp{thm:general_state_complexity}, we know that $\confs(\bm V_n) = 2^{O(n)}$. We now show the lower bound.

We pick weights $\theta$ in the encoder such that $\bv v_t = \bv x_t$. Thus, $\confs(\bv v_t^\theta) = \abs{\Sigma}$ for all $t$. Since the values at each time step are independent, we know that
\begin{align}
    \confs(\bm V_n^\theta) &= \abs{\Sigma}^n \\
    \therefore \confs(\bm V_n) &= 2^{\Omega(n)} .
\end{align}
\end{proof}

\begin{lemma}[\autoref{thm:summary_complexity} restated]
The attention summary vector has state complexity
\begin{equation*}
    \confs(\bv h_n) = O(n^2) .
\end{equation*}
\end{lemma}

\begin{proof}
By \autoref{thm:asymptotic_attention}, we know that
\begin{equation}
    \limN \bv h_n = \limN \frac{1}{m} \sum_{i=1}^m \bv v_{t_i}.
\end{equation}
\noindent By construction, there is a finite set $S$ containing all possible configurations of every $\bv v_t$. We bound the number of configurations for each $\bv v_{t_i}$ by $\abs{S}$ to get
\begin{equation}
    \confs(\bv h_n) \leq \sum_{m=1}^n\abs{S} m \leq \abs{S} n^2 = O(n^2) .
\end{equation}
\end{proof}

\begin{lemma}[Attention state complexity lower bound]
The attention summary vector has state complexity
\begin{equation*}
    \confs(\bv h_n) = \Omega(n) .
\end{equation*}
\end{lemma}

\begin{proof}
Consider the case where keys and values have dimension $1$. Further, let the input strings come from a binary alphabet $\Sigma = \{0, 1\}$. We pick parameters $\theta$ in the encoder such that, for all $t$,
\begin{equation}
    \limN v_t =
    \begin{cases}
    0 & \textrm{if} \; \bv x_t = \bv 0 \\
    1 & \textrm{otherwise}
    \end{cases}
\end{equation}
\noindent and $\limN k_t=1$. Then, attention returns
\begin{equation}
    \limN \sum_{t=1}^n v_t = \frac{l}{n}
\end{equation}
\noindent where $l$ is the number of $t$ such that $\bv x_t= \bv 1$. We can vary the input to produce $l$ from $1$ to $n$. Thus, we have
\begin{align}
    \confs(\bv h_n^\theta) &= n \\
    \therefore \confs(\bv h_n) &= \Omega(n) .
\end{align}
\end{proof}

\begin{lemma}[Attention state complexity with unique maximum]
If, for all $\bm X$, there exists a unique $t^*$ such that $t^* = \max_t \bv q_n \bv k_t$, then
\begin{equation*}
    \confs(\bv h_n) = O(1) .
\end{equation*}
\end{lemma}

\begin{proof}
If $\bv q_n \bv k_t$ has a unique maximum, then by \autoref{cor:injective_attention} attention returns
\begin{equation}
    \limN \arg \max_{\bv v_t} \bv q \bv k_t = \limN \bv v_{t^*}.
\end{equation}
\noindent By construction, there is a finite set $S$ which is a superset of the configuration set of $\bv v_{t^*}$. Thus,
\begin{equation}
    \confs(\bv h_n) \leq \abs{S} = O(1).
\end{equation}
\end{proof}

\begin{lemma}[Attention state complexity with ReLU activations] \label{thm:attention_infinite_values}
If $\limN \bv v_t \in \{0, \infty \}^k$ for $1 \leq t \leq n$, then
\begin{equation*}
    \confs(\bv h_n) = O(1) .
\end{equation*}
\end{lemma}

\begin{proof}
By \autoref{thm:asymptotic_attention}, we know that attention computes
\begin{equation}
    \limN \bv h_n = \limN \frac{1}{m} \sum_{i=1}^m \bv v_{t_i} .
\end{equation}
\noindent This sum evaluates to a vector in $\{ 0, \infty \}^k$, which means that
\begin{equation}
    \confs(\bv h_n) \leq 2^k = O(1) .
\end{equation}
\end{proof}

\autoref{thm:attention_infinite_values} applies if the sequence $\bv v_1, .., \bv v_n$ is computed as the output of $\ReLU$. A similar result holds if it is computed as the output of an unsquashed linear transformation.

\section{CNN Lemmas}

\begin{lemma}[CNN counterexample] \label{thm:cnn_counterexample}
\begin{equation*}
    a^*ba^* \notin L(\CNN) .
\end{equation*}
\end{lemma}

\begin{proof}
By contradiction. Assume we can write a network with window size $k$ that accepts any string with exactly one $b$ and reject any other string. Consider a string with two $b$s at indices $i$ and $j$ where $\abs{i - j} > 2k + 1$. Then, no column in the network receives both $\bv x_i$ and $\bv x_j$ as input. When we replace one $b$ with an $a$, the value of $\bv h_+$ remains the same. Since the value of $\bv h_+$ \eqref{eq:pooling} fully determines acceptance, the network does not accept this new string. However, the string now contains exactly one $b$, so we reach a contradiction.
\end{proof}

\begin{definition}[Strictly $k$-local grammar]
A strictly $k$-local grammar over an alphabet $\Sigma$ is a set of allowable $k$-grams $S$. Each $s \in S$ takes the form
\begin{equation*}
    s \in \big( \Sigma \cup \{ \# \} \big)^k
\end{equation*}
\noindent where $\#$ is a padding symbol for the start and end of sentences.
\end{definition}

\begin{definition}[Strictly local acceptance]
A strictly $k$-local grammar $S$ accepts a string $\sigma$ if, at each index $i$,

\begin{equation*}
    \sigma_i \sigma_{i+1} .. \sigma_{i + k - 1} \in S .
\end{equation*}

\end{definition}

\begin{lemma}[Implies \autoref{thm:conv_strictly_local}] \label{thm:strictly_local_reduction}
A $k$-CNN can asymptotically accept any strictly $2k + 1$-local language.
\end{lemma}

\begin{proof}
We construct a $k$-CNN to simulate a strictly $2k + 1$-local grammar. In the convolutional layer \eqref{eq:conv}, each filter identifies whether a particular invalid $2k + 1$-gram is matched. This condition is a conjunction of one-hot terms, so we use $\tanh$ to construct a linear transformation that comes out to $1$ if a particular invalid sequence is matched, and $-1$ otherwise.

Next, the pooling layer \eqref{eq:pooling} collapses the filter values at each time step. A pooled filter will be $1$ if the invalid sequence it detects was matched somewhere and $-1$ otherwise.

Finally, we decide acceptance \eqref{eq:ff} by verifying that no invalid pattern was detected. To do this, we assign each filter a weight of $-1$ use a threshold of $-K + \frac{1}{2}$ where $K$ is the number of invalid patterns. If any filter has value $1$, then this sum will be negative. Otherwise, it will be $\frac{1}{2}$. Thus, asymptotic sigmoid will give us a correct acceptance decision.
\end{proof}

\section{Neural Stack Lemmas}

Refer to \citet{hao2018context} for a definition of the StackNN architecture. The architecture utilizes a differentiable data structure called a \textit{neural stack}. We show that this data structure has $2^{\Theta(n)}$ state complexity.

\begin{lemma}[Neural stack state complexity] \label{thm:stack_state_complexity}
Let $\bm S_n \in \mathbb{R}^{nk}$ be a neural stack with a feedforward controller. Then,

\begin{equation*}
    \confs(\bm S_n) = 2^{\Theta(n)} .
\end{equation*}
\end{lemma}

\begin{proof}
By the general state complexity bound \autorefp{thm:general_state_complexity}, we know that $\confs(\bm S_n) = 2^{O(n)}$. We now show the lower bound.

The stack at time step $n$ is a matrix $\bm S_n \in \mathbb{R}^{nk}$ where the rows correspond to vectors that have been pushed during the previous time steps. We set the weights of the controller $\theta$ such that, at each step, we pop with strength $0$ and push $\bv x_t$ with strength $1$. Then, we have
\begin{align}
    \confs(\bm S_n^\theta) &= \abs{\Sigma}^n \\
    \therefore \confs(\bm S_n) &= 2^{\Omega(n)} .
\end{align}
\end{proof}

\end{document}